\documentclass{ecai}  

\usepackage{graphicx}
\usepackage{latexsym}
\usepackage{commands}
\usepackage{amsmath}
\usepackage{todonotes}
\usepackage{amsthm}
\usepackage{float}
\usepackage{amsfonts}
\usepackage{algorithm}
\usepackage{algpseudocode}
\usepackage{xcolor}
\usepackage{listings}

\definecolor{codegreen}{rgb}{0,0.6,0}
\definecolor{codegray}{rgb}{0.5,0.5,0.5}
\definecolor{codepurple}{rgb}{0.58,0,0.82}
\definecolor{backcolour}{rgb}{0.95,0.95,0.92}

\lstdefinestyle{mystyle}{
    backgroundcolor=\color{backcolour},   
    commentstyle=\color{codegreen},
    keywordstyle=\color{magenta},
    numberstyle=\tiny\color{codegray},
    stringstyle=\color{codepurple},
    basicstyle=\ttfamily\footnotesize,
    breakatwhitespace=false,         
    breaklines=true,                 
    captionpos=b,                    
    keepspaces=true,                 
    numbers=left,                    
    numbersep=5pt,                  
    showspaces=false,                
    showstringspaces=false,
    showtabs=false,                  
    tabsize=2
}

\renewcommand{\phi}{\varphi}

\newtheorem{example}{Example}
\newtheorem{proposition}{Proposition}
\newtheorem{definition}{Definition}

\newtheorem{theorem}{Theorem}
\newtheorem*{theorem*}{Theorem}

\NewEnviron{cproblem}[1]{%
\smallskip
{\centering \fbox{\parbox{0.97\columnwidth}{%
    {\centering\scshape #1\par}%
    \parskip=1ex
    \BODY
}}}\smallskip}

\begin{document}

\begin{frontmatter}

\title{Anticipating Responsibility in Multiagent Planning}

\author[A]{\fnms{Timothy}~\snm{Parker}\thanks{Corresponding Author. Email: timothy.parker@irit.fr}}
\author[A]{\fnms{Umberto}~\snm{Grandi}}
\author[A]{\fnms{Emiliano}~\snm{Lorini}}

\address[A]{IRIT, CNRS, University of Toulouse, France}

\begin{abstract}
Responsibility anticipation is the process of determining if the actions of an individual agent may cause it to be responsible for a particular outcome. This can be used in a multi-agent planning setting to allow agents to anticipate responsibility in the plans they consider. The planning setting in this paper includes partial information regarding the initial state and considers formulas in linear temporal logic as positive or negative outcomes to be attained or avoided. We firstly define attribution for notions of active, passive and contributive responsibility, and consider their agentive variants. We then use these to define the notion of responsibility anticipation. We prove that our notions of anticipated responsibility can be used to coordinate agents in a planning setting and give complexity results for our model, discussing equivalence with classical planning. We also present an outline for solving some of our attribution and anticipation problems using PDDL solvers.
\end{abstract}
\end{frontmatter}

\section{Introduction}

In any multi-agent setting, a key concept is that of responsibility. There are two main notions of responsibility, which are forward-looking and backward-looking responsibility \cite{Poel2011}. In general, forward-looking responsibility is to have an obligation to bring about or prevent a certain state of affairs, while backward-looking responsibility means to be held accountable for a particular action or state of affairs that occurred. Our paper considers only backward-looking responsibility, which is often used in multi-agent settings to determine appropriate sanctions or rewards for agents. While responsibility attribution is a well-studied problem \cite{AlechinaHL17,Baier0M21,BrahamVanHees,HalpernK18,Naumov021}, we focus on the novel concept of responsibility anticipation, which means to determine if a particular plan for a single agent \textit{may} lead to their responsibility for some outcome, given the possible plans of all other agents. We believe that by anticipating responsibility, agents will be better able to coordinate their actions even if they cannot communicate. We consider responsibility in a multi-agent setting with concurrent actions and where outcomes are described in Linear Temporal Logic over finite traces ($\ltllogic$)\cite{GiacomoV13}. Following the work of Lorini Et Al \cite{LoriniLM14} we recognise two key components to responsibility, namely the causal and agentive components. The causal component requires that the actions of the agent in some way contributed to the outcome in question. Lorini Et Al identify two different notions of causal responsibility, active and passive responsibility. We formalise both in our model as well as a notion of contributive responsibility defined by Braham and Van Hees \cite{BrahamVanHees}. Roughly speaking, given some state of affairs $\omega$, active responsibility means to bring about $\omega$, passive responsibility means to allow $\omega$ to occur, and contributive responsibility means to be part of a coalition that brings about $\omega$.  The agentive component requires that the agent is aware that their actions will (or in some cases may) contribute to the outcome. In our setting the agents have full knowledge of the action theory (i.e the capabilities of all agents), but are uncertain regarding the intended actions of other agents and the initial state of the world. This allows us to define agentive notions of active, passive and contributive responsibility.

While our model allows us to attribute responsibility retrospectively (after plan execution), the focus of our work is in anticipating responsibility to aid in plan selection for a single agent. Since agents often cannot be certain about the outcomes of their plans, we introduce a notion of anticipated responsibility, which can be applied to any of our previous notions of responsibility. We show that by minimising their anticipated responsibility for a negative outcome, agents are often capable of guaranteeing that the outcome does not occur, even in some cases where the agents cannot communicate and where no single agent can guarantee avoiding the negative outcome.

We intend for our model to be useful in real-world planning applications. This is why we have taken efforts to ensure that the our planning domain is reasonably compact while still being highly expressive. We also outline how our responsibility attribution and anticipation problems can be reduced to PDDL, both to demonstrate how pre-existing planning solvers can be applied to our problems and to encourage implementation of our model.

Our paper is organised as follows. Section \ref{relatedwork} situates our paper with reference to related work in responsibility attribution, and compares our work to several similar papers. Section \ref{sec:model} introduces our multi-agent planning domain and presents an explanatory example. Section \ref{sec:resp} formalises our notions of responsiblity attribution and anticipation and discusses their application to multi-agent planning. Section \ref{sec:theorems} gives the complexity results for our setting and an outline of a reduction to PDDL. Finally section \ref{sec:futurework} summarises the paper and outlines directions for future work.

\section{Related Work}\label{relatedwork}

This work contributes primarily to the field of formalised responsibility attribution. It also involves planning with temporally extended goals \cite{44,degiacomo2015,47}, but since we are not aware of any other work in planning that considers responsibility in plan selection, we will focus this section on responsibility. Our planning model builds on a number of previous papers which are discussed in section \ref{sec:model}.
 
Furthermore, responsibility anticipation and its application to planning agents, is, to the best of our knowledge, also novel in the field of responsibility formalisation. Therefore, we will focus on approaches to responsibility attribution in the literature, and discuss how and why they differ from our work.

One approach to formalising responsibility is the work of Alechina Et Al \cite{AlechinaHL17}, which is based on work by Chockler and Halpern \cite{ChocklerH04,Halpern15} on the formalisation of responsibility. Rather than using $\ltllogic$, as in our approach, this work uses structural equation modeling (SEM). Their paper focuses specifically on responsibility attribution for the failure of a previously-arranged joint plan, which is a specific sequence of tasks that all agents are expected to follow (but perhaps will not), making its application much more specific than our work. Unlike our model, the authors focus only on a single notion of responsibility, but it does model varying degrees of responsibility for different agents. Alechina Et Al also perform a complexity analysis of their model, showing that responsibility attribution is in general NP-Complete (in line with our notion of passive responsibility, see theorem \ref{thm:attributioncomplexity}) and identify some fragments where responsibility attribution is polynomial.

Halpern and Kleiman-Weiner \cite{HalpernK18} also use a structural equations model, but focuses on defining the intentions of agents given their actions and their epistemic state (given here as a probability distribution). As this paper does not address causal responsibility there is not much overlap with our model, but it does highlight several interesting concepts that we could attempt to incorporate in future work.

A more general but less compact approach is the work of Baier Et Al \cite{Baier0M21}. Their work covers both forward and backward-looking responsibility attribution, but we will focus on their formalisation of backward-looking responsibility. Whereas our model is based of classical planning, Baier Et Al use extensive form games with strategies instead of plans. This makes their model much less compact and more complex than ours, but also more expressive. In their work, a coalition of agents $J$ is causal backwards responsible for some outcome $\omega$ if fixing the strategies of all other agents (and the random choices of Nature) there exists a strategy for $J$ where $\omega$ does not occur in any possible execution. They also define strategic backward responsibility, which states that $\omega$ occurs, and there is some state in the execution where the coalition of agents have a strategy such that $\omega$ does not occur in any epistemically possible outcome for that strategy (since agents cannot distinguish between some states). Again, this model does not include any other notions of responsibility, but does model an agents degree of responsibility, which is determined by an agents membership to one or more responsible coalitions, meaning it behaves similarly to our notion of contributive responsibility, though defined on strategies instead of plans. Baier Et Al also provide a complexity result for their model. They note that the complexity of responsibility attribution is in NP, which is a lower bound than contributive responsibility in our model (see theorem \ref{thm:attributioncomplexity}), though our model is exponentially more compact.

A similar definition of responsibility exists in the work of Naumov and Tao \cite{Naumov021} whose setting of Imperfect Information Strategic Games is very close to our notion of planning domain, but restricted to plans of length 1. Their notion of blameworthiness says that $i$ is blameworthy for $\omega$ if $\omega$ occurrs and $i$ could have performed an action guaranteeing $\neg \omega$ in all possible states. This is a stronger version of our notion of causal passive responsibility, as we require only that $i$ could have avoided $\omega$ if the state and all actions of other agents were fixed. They also present a notion of ``seeing to it'' which requires that an agent guarantees in all possible worlds that $\omega$ occurs. This is very close to our notion of agentive active responsibility, the only difference being that in our model there must exist some possible history from the initial state where the outcome does not occur, whereas in their model that history can start at any epistemically possible state for $i$. Also, unlike us Naumov and Tao formalise their notions as operators in logic, allowing for the development of a proof system for these operators (they develop a proof system for their notion of blameworthiness in a previous, perfect-information setting \cite{NaumovT19}).

Our work is heavily inspired by the work of Lorini Et Al \cite{LoriniLM14}. This paper formalises the notions of active and passive responsibility that we use in this paper, as well as the variant of agentive responsibility. The model in this paper is based on STIT logic in a multi-agent setting with Kripke possible worlds. We extend this work to the setting of multi-agent planning, though for simplicity we do not model agents having knowledge of the possible actions of other agents, in our setting all plans of the other agents are considered possible.

Our work is also related to the work of Braham and van Hees \cite{BrahamVanHees}, who analyse responsibility in a game-theoretic framework. One of the conditions for moral responsibility is that an agent's actions must have ``causally contributed'' to the outcome in question. We adapt the notion of causal contribution into our setting as a third notion of causal responsibility.

\section{Model}\label{sec:model}

In this section we introduce the planning framework in which we will define our notions of responsibility. As many of our definitions are drawn from existing literature, in the interests of space we have chosen to omit some of the less informative formal definitions, which can be found in the supplementary material for this paper. We will indicate where we have done this.

\subsection{Agents, Actions and Histories}

The building blocks of our model are a finite set of agents $\agentset$ and a countable set of propositions $\propset = \{p, q, \ldots\}$. From $\propset$ we define a set of states $\stateset = 2^{\propset }$, with elements $s,s', \ldots$ Let $\actset = \{a, b , \ldots\}$ be a finite non-empty set of action names.

To trace the actions of agents and changing states over time we define a $k$-history to be a pair $\history= (\historyst, \historyact)$ with $\historyst: [0,k] \longrightarrow \stateset$ and $\historyact: \agentset \times [1,k] \longrightarrow \actset$. The set of $k$-histories is noted $\historyset{k}$. The set of all histories is $\historyset{}=\bigcup_{k \in \nat }\historyset{k}$.

\subsection{Multi-Agent Action Theory}\label{sec:actions}

Given the actions performed by an agent, we need to be able to determine the effects of those actions. We favour a compact action theory based on situation calculus \cite{sitcalc}.

We first define $\langlogic_{\proplogic+}$ (propositional logic with action descriptions) as follows:
\begin{center}\begin{tabular}{lcl}
  $\phi $ &  $\bnf$ & $ p \mid do(i,a)  \mid  \neg\phi \mid \phi  \wedge \phi$
\end{tabular}\end{center}
with 
$p$ ranging  over $\propset$, $i$ ranging over $\agentset$ and $a$ ranging over $\actset$. Atomic formulas in this language are those that consist of a single proposition $p$ or a single instance of $do(i,a)$.

Semantic
interpretation of
formulas in
$\langlogic_{\proplogic+}$ is performed
relative to a $k$-history 
$\history \in \historyset{}$ and a time point
$t \in \{0,\ldots,k\}$ and as follows (we omit
boolean cases which are defined as usual):
\begin{alignat*}{2}
  \history,  t &\models  p  & ~\IFF~ & 
  p \in \historyst(t) , \\
  \history, t &\models do(i,a) & ~\IFF~ & t < k \AND \historyact(i,t) = a
\end{alignat*}
We define our action theory as a pair of a positive and negative effect precondition function
$\effprecond=(\effprecondplus,\effprecondminus)$, where $ \effprecondplus:
\agentset \times \actset \times \propset
\longrightarrow \langlogic_{\proplogic}$
and  $\effprecondminus :
\agentset \times \actset \times \propset
\longrightarrow \langlogic_{\proplogic}$. If the formula $\effprecondplus( i,a,p )$ 
holds in a state where action $a$ is executed by agent $i$, proposition $p$
will be \emph{true} in the next state (provided no other action interferes). Similarly, 
$\effprecondminus(i,a,p)$ 
guarantees that $p$ will be \emph{false} in the next state if action $a$ is executed by $i$ (without interference).
In case of conflicts between actions, we use an intertial principle: if two or more actions attempt to enforce different truth values for $p$, then the truth value of $p$ does not change.

If we want to signal that action $a$ is not available to agent $i$ we can simply set $\effprecondplus(i,a,p) = \effprecondminus(i,a,p) = \bot$ for all $p \in \propset$. We assume the existence of a ``do nothing'' action $\skipact$, defined such that $\effprecondplus(i,\skipact,p) = \effprecondminus(i, \skipact, p) = \bot$ for all $i$ and $p$.

We say that history $\history$ is a $\effprecond$-compatible history for action theory $\effprecond=(\effprecondplus,\effprecondminus)$ if each state respects the actions performed in the previous state. The set of $\effprecond$-compatible histories is noted $\historyset{}(\effprecond)$. A full formal definition can be found in the supplementary material. 

\subsection{Compactness of our Action Theory}

A conceptually simpler equivalent to our notion of action theory is a state transition function \cite{Keller76} $\tau: \stateset \times \actset^\agentset \longrightarrow \stateset$. This takes as input the current state and the actions of all agents, and outputs the next state. Since there are no limitations on what states the function can output (besides functionality), it is straightforward to see that for any deterministically consistent history $\history$ (meaning the same joint action in the same state always leads to the same outcome), there is some state transition function $\tau$ that can be used to generate $\history$ given the start state and the actions of all agents.

However, we can show that our action theory is equally as expressive as any state transition function, and strictly more succinct, and this is achieved by our use of action descriptions.

\begin{proposition}
Given a state transition function $\tau$, there exists an action theory $\effprecond$ that is equivalent to (generates the same histories as)
$\tau$ and is at worst polynomially larger in size.
\end{proposition}

\begin{proposition}
There exists some state transition function $\tau_1$ such that any action theory $\effprecond$ that is equivalent to $\tau_1$ must contain $\does{i}{a}$ in its description.
\end{proposition}

Note that the size of $\tau$ is always exponential in the size of $\propset$ and $\agentset$, since the number of entries in $\tau$ are fixed. On the other hand, entries for $\effprecond$ can in be as small as constant size (for example $\effprecondplusminus(i,a,p) \in \{\top, \bot\}$). This means $\effprecond$ can be as small as $2 \times |\actset| \times |\agentset| \times |\propset|$. We conjecture that in most applications for this planning model, the action theory $\gamma$ will be polynomial in size in $\propset$, $\agentset$ and $\actset$.

\subsection{Planning Domains with Partial Information}

We can now define our notion of planning domain. For now, we simply define a space where agents can create and execute plans, and where the outcomes of those plans can be determined. Since our planning domain includes partial information we make use of epistemic equivalence sets. An epistemic equivalence set $\epequiv_i \subseteq \stateset$ is the set of possible start states from the perspective of agent $i$.

\begin{definition}[Partial Information Multi-Agent Planning Domain]
A Partial information multi-agent Planning Domain (PPD) is a tuple $\pldomain=(\effprecond, s_0,(\epequiv_i)_{i\in \agentset})$ where $\effprecond=(\effprecondplus,\effprecondminus)$ is an action theory, $s_0$ is an initial state, and for each $i \in \agentset$, $\epequiv_i$ is the epistemic equivalence set for $i$.
\end{definition}

Our notion of an epistemic equivalence sen, is straightforward and very general, but not very compact. A more compact alternative would be to give each agent visibility of a certain subset of the propositions in $\propset$ \cite{Torreno15}. However, this would be less general as not all epistemic equivalence sets can be expressed in terms of visibility. A more complex but much more general approach would be to give each agent a belief base $\belequiv_i$ as a set of formulas of $\langlogic_{\proplogic}$ that describes the beliefs of $i$ regarding the initial state \cite{Lorini20Epistemic}. We prefer epistemic equivalence sets as this is the simplest notion for defining algorithms. Furthermore, any of the above methods will induce an epistemic equivalence set, meaning our model can easily be adapted to other systems.

\begin{example}[Crossing a Junction]\label{ex:joint}

The planning domain $\pldomain_E$ models an autonomous vehicle (Agent 1) approaching a junction. Agent 1 knows that there is a second vehicle (Agent 2) near the junction, but does not know if Agent 2 has crossed the junction. Each vehicle can either go straight on ($\actF$), or do nothing ($\skipact$).
\begin{figure}[h]
\begin{center}
\includegraphics[scale = 0.2]{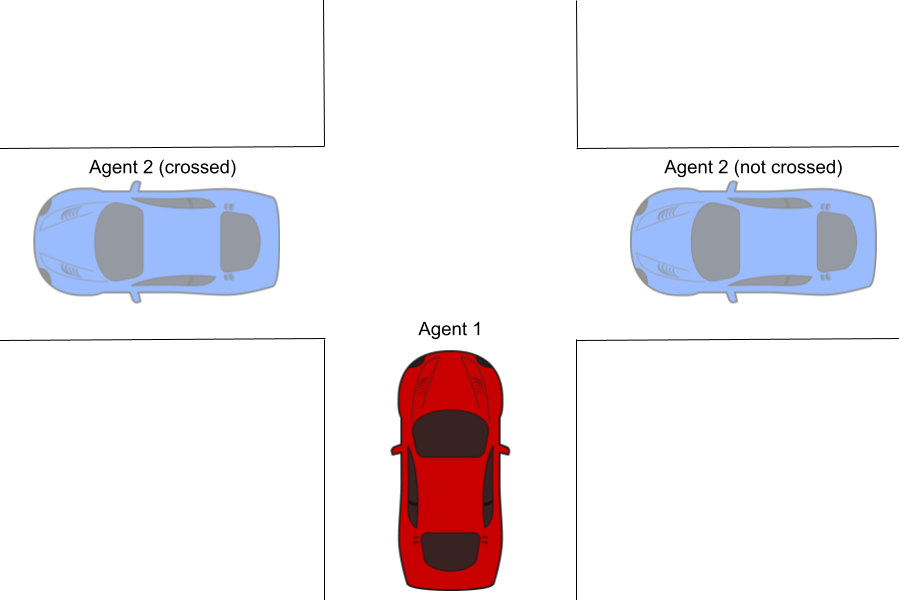}
\caption{A visual representation of $\pldomain_E$ (Example~\ref{ex:joint}), showing the start position of Agent 1 and the two possible positions of Agent 2 (crossed or not crossed the junction).}\label{figure2}
\end{center}
\end{figure}

The example is formally defined as follows:
\begin{itemize}
    \item $\agentset = \{A1, A2\}$
    \item $\propset = \{\crossed{1}, \crossed{2}, \iscollision\}$
    \item $\actset = \{\actF, \skipact\}$
    \item $s_0 = \emptyset$, $s_1 = \{\crossed{2}\}$
    \item $(\epequiv_{1})  = \{s_0, s_1\}$
\end{itemize}

The action theory for our example is defined as follows, note that we have already defined the preconditions for $\skipact$ in section \ref{sec:actions}:
\begin{align*}
    \effprecondplus(A1, \actF,\crossed{1}) = &\neg (\neg \crossed{2} \land \does{A2}{\actF})\\
    &\land \neg \iscollision\\
    \effprecondplus(A1, \actF, \iscollision) = &\neg \crossed{1} \land \neg \crossed{2}\\
    &\land \does{A2}{\actF}\\
    \effprecondplus(A2, \actF,\crossed{2}) = &\neg (\neg \crossed{1} \land \does{A1}{\actF})\\
    &\land \neg \iscollision\\
    \effprecondplusminus(i, \actF,p) = &\bot \text{ unless stated otherwise above.}
\end{align*}

In words, if exactly one agent attempts to cross the junction ($\actF$) then they will succeed. If both agents perform $\actF$ at the same time then they will collide, which will prevent either from being able to move.
\end{example}


 \subsection{Action Sequences and Joint Plans}

Now that we have defined a planning domain, we can define the notions of action sequence and plan.
Given $k \in \nat$, a $k$-action-sequence is a function 
\begin{align*}
    \actseq: \{0, \ldots, k-1\} \longrightarrow \actset.
\end{align*}
The set of $k$-action-sequences is noted $\seqset{k}$. The set of all action sequences is $\seqset{}=\bigcup_{k \in \nat }\seqset{k}$. For a (non-empty) coalition of agents $J \in 2^{\agentset} \setminus \emptyset$ we can define a joint $k$-plan as a function $\plan: J \longrightarrow \seqset{k}$ (if $J$ is a singleton coalition we call $\plan$ an individual plan). The set of joint $k$-plans for a coalition $J$ is written $\jointplanset{k}{J}$. The set of all joint plans for $J$ is $\jointplanset{}{J} = \bigcup_{k \in \nat } \jointplanset{k}{J}$.

Given a joint plan $\plan$ for coalition $J$ and another coalition $J' \subseteq J$, we can write the sub-plan of $\plan$ corresponding to $J'$ as $\plan^{J'}$, we can also write $\plan^{-J'}$ for sub-plan corresponding to $J\setminus J'$. Given two $k$-plans $\plan_1$ and $\plan_2$ for disjoint coalitions $J_1, J_2$, we write $\plan_1 \cup \plan_2$ for the joint plan for $J_1 \cup J_2$ such that $(\plan_1 \cup \plan_2)^{J_1} = \plan_1$ and $(\plan_1 \cup \plan_2)^{J_2} = \plan_2$. Finally, given two plans $\plan_1$ and $\plan_2$, if there exists some plan $\plan_3$ such that $\plan_2 = \plan_1 \cup \plan_3$ then we say that $\plan_1$ \textit{is compatible with} $\plan_2$.

We can now define the notion
of the history generated by a joint $k$-plan $\plan$ at an initial
state $s_0$ under the action theory $\effprecond$. It is the $\effprecond$-compatible $k$-history 
along which the agents jointly execute the plan
$\plan$ starting at state $s_0$. We write this as $\history^{\plan,s_0,\effprecond}$

\subsection{Linear Temporal Logic}

In our model histories are temporal entities that are always finite in length, therefore the most natural choice to describe properties of histories is Linear Temporal Logic over Finite Traces \cite{GiacomoV13,degiacomo2015}. This allows us to describe temporal properties such as ``$\phi$ never occurs'' or ``$\phi$ always occurs immediately after $\psi$''. We write the language as $\langlogic_{\ltllogic}$, defined by the following grammar:
\begin{center}\begin{tabular}{lcl}
  $\phi $ &  $\bnf$ & $ p \mid do(i,a)  \mid  \neg\phi \mid \phi  \wedge \phi   \mid \nexttime \phi \mid
\until { \phi   } { \phi    },  $\\
\end{tabular}\end{center}
with 
$p$ ranging  over $\propset$, $i$ ranging over $\agentset$ and $a$ ranging over $\actset$. Atomic formulas in this language are those that consist of a single proposition $p$ or a single instance of $do(i,a)$.
$\nexttime$
and $\until {     } {     }$
are the  operators
``next''
and ``until'' of $\ltllogic$. 
Operators
``henceforth'' ($\henceforth$)
and ``eventually'' ($\eventually$)
are defined in the usual way:
$\henceforth \phi \defin \neg ( \until {   \top   } {  \phi     }) $
and 
$\eventually \phi  \defin \neg  \henceforth \neg  \phi $. We define the semantics for $\nexttime$ and $\until{}{}$ as follows, the rest is the same as $\langlogic_{\proplogic+}$.
\begin{alignat*}{2}
  \history,  t &\models \nexttime \phi & ~\IFF~ & t < k \AND \history,  t+1 \models   \phi, \\
   \history,  t &\models \until { \phi_1    } { \phi_2    }   & ~\IFF~ &
    \begin{aligned}[t]
      &\exists t' \geq t :   t' \leq k \AND \history,  t' \models \phi_2 \AND\\
      &\forall t''   \geq t  :   \IF 
      t'' < t' ~\THEN~ \history,  t'' \models \phi_1.
    \end{aligned} 
\end{alignat*}
\section{Formalising Responsibility}\label{sec:resp}

In order to define responsibility anticipation, we must first define the responsibility attribution. Responsibility attribution is a backward-looking notion where, given some fixed history, we seek to determine which agents are responsible for some particular outcome. We distinguish between ``agentive'' and merely ``causal'' forms of responsibility. For an agent $i$ to be causally responsible for some outcome $\omega$ simply means that the actions of $i$ were in some way a causal factor in the occurrence of $\omega$. Agentive responsibility requires the additional condition that $i$ \textit{knew} that its actions could or would lead to $\omega$.

Another common notion of responsibility is that of moral responsibility, which is the kind of responsibility that typically merits praise or blame. We do not attempt to formalise moral responsibility in this paper as it is an extremely complex notion, and there is widespread disagreement in the literature regarding exactly what the criteria for moral responsibility are \cite{sep-moral-responsibility}. That said, we do believe that agentive responsibility is a necessary (but not sufficient) condition for moral responsibility. 

\subsection{Causal Responsibility}

To be causally responsible for an outcome roughly means to have causally contributed to that outcome occurring. Two main notions of causal responsibility are active and passive responsibility. To be actively responsible means to directly cause the outcome, i.e to act in a way that guarantees the outcome will occur. To be passively responsible means to allow an outcome to occur while having the ability to prevent it. Our definitions of active and passive responsibility are based on the work of Lorini Et Al \cite{LoriniLM14}, but adapted for a multi-agent planning domain.

\begin{definition}[Active Responsibility]\label{def:a-blame}
Let $\pldomain=(\effprecond, s_0,(\epequiv_i)_{i\in \agentset})$ be a PPD, $i \in \agentset$ an agent, and $\plan_1$ a joint plan. Let $\omega \in \langlogic_{\ltllogic}$. Then, we say that $i$ bears \emph{Causal Active Responsibility} (CAR) for $\omega$ in $(\plan_1,s_0,\effprecond)$, if $\history^{\plan_2,s_0,\effprecond} \models \omega$ for all $\plan_2$ compatible with $\plan_1^{\{i\}}$ and there exists some joint plan $\plan_3 \in \jointplanset{}{\agentset}$ such that $\history^{\plan_3,s_0,\effprecond} \not \models \omega$.
\end{definition}

Where $s_0$ and/or $\effprecond$ are obvious from context, they are omitted from the statement ``$i$ bears CAR for $\omega$ in $(\plan_1,s_0,\effprecond)$'' In words, an agent $i$ is causally actively responsible for the occurrence of $\omega$ if, keeping fixed the initial state and the actions of $i$, the other agents could not have acted differently and prevented the occurrence of $\omega$. Note that active responsibility requires that the outcome does not occur in all possible plans. This means that an agent cannot be actively responsible for something that was inevitable, such as the sun rising in the morning. This corresponds to the deliberative STIT operator of Horty and Belnap \cite{HortyDstit}.

\begin{definition}[Passive Responsibility]\label{def:p-blame}
Let $\pldomain = (\effprecond, s_0, (\epequiv_i)_{i\in \agentset})$ be a PPD, $i \in \agentset$ an agent, and $\plan_1$ a joint plan. Let $\omega \in \langlogic_{\ltllogic}$. Then, we say that $i$ bears \emph{Causal Passive Responsibility} (CPR) for $\omega$ in $(\plan_1,s_0,\effprecond)$ if $\history^{\plan_1,s_0,\effprecond} \models \omega$ and there exists some $\plan_2$ compatible with $\plan_1^{-\{i\}}$ such that $\history^{\plan_2,s_0,\effprecond} \not \models \omega$.
\end{definition}

An agent $i$ is passively responsible for some outcome $\omega$ if, keeping fixed the initial state and the actions of all other agents, it could have acted differently and prevented the occurrence of $\omega$. 

Passive and active responsibility can fail in cases of causal overdetermination. For example: suppose three men push a car off a cliff. Since the car is heavy, two of them are needed to successfully push the car, meaning no agent is actively responsible. Since any one man could have stopped pushing without changing the outcome, no man is passively responsible. Nonetheless it intuitively seems that each man is at least somewhat responsible. Therefore, we introduce the notion of contributive responsibility based on the work of Braham and van Hees \cite{BrahamVanHees}, which is a more general notion of causal responsibility.

\begin{definition}[Contributive Responsibility]\label{def:c-blame}
Let $\pldomain = (\effprecond, s_0, (\epequiv_i)_{i\in \agentset})$ be a PPD, $i \in \agentset$ an agent, and $\plan_1$ a joint plan. Let $\omega \in \langlogic_{\ltllogic}$. Then, we say that $i$ bears \emph{Causal Contributive Responsibility} (CCR) for $\omega$ in $(\plan_1,s_0,\effprecond)$ if $\history^{\plan_1,s_0,\effprecond} \models \omega$ and there exists some coalition of agents $J$ such that $i \in J$ and for all $\plan_2$ compatible with $\plan_1^J$, $\history^{\plan_2,s_0,\effprecond} \models \omega$ and there exists some $\plan_3$ compatible with $\plan_1^{J \setminus \{i\}}$ such that $\history^{\plan_3,s_0,\effprecond} \not \models \omega$.
\end{definition}

In words, an agent $i$ is contributively responsible for $\phi$ if it is part of some coalition of agents $J$ such that:
a) the actions of $J$ were sufficient to guarantee $\phi$; and b) the actions of $J \setminus \{i\}$ were not sufficient to guarantee $\phi$. In terms of STIT this can be written as ``$\exists J \subseteq \agentset \suchthat i \in J \land STIT_J \omega \land \neg STIT_{J\setminus\{i\}}\omega$''.


A notable property of Causal Contributive Responsibility is that it is ``complete''. This means that for any outcome that occurs in a plan, either that outcome was inevitable or there is at least one agent who is responsible (i.e bears CCR) for that outcome.

\begin{theorem}\label{thm:completeness}
Let $\pldomain=(\effprecond, s_0,(\epequiv_i)_{i\in \agentset})$ be a PPD, let $\plan$ be a joint plan and let $\history = \history^{\plan,s_0,\effprecond}$. Let $\omega \in \langlogic_{\ltllogic}$ such that $\history \models \omega$. Then either $\history' \models \omega$ for every history compatible with $\pldomain$, or there exists some $i \in \agentset$ such that $i$ bears CCR for $\omega$ in $\plan$.
\end{theorem}

Another important property of our notions of responsibility is that no agent can be held causally responsible (for any form of causal responsibility) for an outcome that was inevitable (i.e occurs in every possible joint plan). This is because all three notions of responsibility require the existence of a joint plan where $\omega$ does not occur.


\subsection{Agentive Responsibility}

To bear agentive responsibility for an outcome, an agent must know that their actions will (or in some cases may) be causally responsible for the outcome occurring. Specifically, we consider the epistemic state of the agent where they have decided their own actions, but do not yet know the actions of others.

\begin{definition}[Agentive Active Responsibility]\label{def:m-blame}
Let $\pldomain=(\effprecond, s_0,(\epequiv_i)_{i\in \agentset})$ be a PPD, $i \in \agentset$ an agent, and $\plan_1$ a joint plan. Let $\omega \in \langlogic_{\ltllogic}$. Then, we say that $i$ bears \emph{Agentive Active Responsibility} (AAR) for $\omega$ in $(\plan_1,s_0,\effprecond)$ if $i$ is actively responsible for $\omega$ in $\plan_1$ and for every $\plan_2$ compatible with $\plan_1^{\{i\}}$, and every $s_1 \in \epequiv_i$, $\history^{\plan_2,s_1,\effprecond} \models \omega$.
\end{definition}

Agent $i$ bears agentive active responsibility for $\omega$ if their actions were sufficient to guarantee $\omega$ in any possible outcome (given the possible start states and possible actions of other agents). Furthermore, as with CAR, there must be some joint plan from $s_0$ where $\omega$ does not occur.

Since passive and contributive responsibility both include the notion of ``allowing'' something to happen rather than ``forcing'' it to happen, the outcome does not need to be guaranteed from the perspective of the agent, but merely possible. This means that the notions of agentive passive and agentive contributive responsibility are both equivalent to their causal definitions, as we assume that agents have full knowledge of the action theory and always consider the true initial state to be epistemically possible, meaning any actual outcome $\omega$ must have been considered possible from the perspective of every agent. Therefore note that the acronyms CPR and CCR refer to both the causal \textit{and} agentive variants of passive and causal responsibility.

A more intuitive notion of agentive passive and contributive responsibility would be to say that $\omega$ must be \textit{reasonably likely} from the perspective of $i$ rather than merely ``possible''. However, since our model contains no notion of probability, plausibility, or knowledge of the actions of other agents, this is not currently possible, though it does present a direction for future iterations of this model.

\begin{example}[Crossing a Junction - continued]
Consider the following joint plan $\plan_1$ from start state $s_0$:
\begin{align*}
A1 : & [1 \mapsto \actF, 2 \mapsto \actF], A2 : [1 \mapsto \actF, 2 \mapsto \actF]
\end{align*}

This will result in a collision. Agent 1 bears CPR (and also CCR) for the negation of the goal requiring that the two cars never collide ($\omega_1 = \henceforth \neg \iscollision$) since in this case A1 could have avoided a collision by waiting for one step before moving (i.e $A1 : [1 \mapsto \skipact, 2 \mapsto \actF]$). However, since Agent 2 also could have waited to avoid a collision, Agent 1 is not actively responsible. Consider an alternative plan where each agent is more cautious:
\begin{align*}
A1 : & [1 \mapsto \skipact, A2 \mapsto \skipact], A2 : [1 \mapsto \skipact, 2 \mapsto \skipact]
\end{align*}
In this case Agent 1 bears CAR and AAR for the negation of the goal that Agent 1 eventually crosses the road ($\omega_2 = \eventually \crossed{1}$), since $\neg \omega_2$ occurs in any history compatible with the actions of Agent 1 in $\plan_2$ starting from $s_0$ or $s_1$.
\end{example}
\subsection{Anticipating Responsibility}\label{sec:anticipating}
Responsibility attribution is defined on known joint plans and known initial states. Therefore it cannot be used in planning for single agents, for whom the actions of the other agents and the initial state are unknown. However, an agent can always know if it is \textit{potentially} responsible for that outcome, namely if there is some possible history compatible with that plan where they are responsible.\footnote{We could also consider anticipation with universal instead of existential quantification, but being responsible in \textit{every} possible history is a very strong notion and we have not found much use for it.}
\begin{definition}[Anticipated Responsibility]
Let $\pldomain = (\effprecond, s_0, (\epequiv_i)_{i \in \agentset})$ be a PPD, $i \in \agentset$ an agent, and $\plan$ an individual plan. Let $\omega \in \langlogic_{\ltllogic}$ and $X$ a form of responsibility (CAR, CPR, CCR, AAR). Then, we say that $i$ anticipates $X$ for $\omega$ in $(\plan,\pldomain)$ if there is some $s_1 \in \epequiv_i$ and some joint plan $\plan_1$ compatible with $\plan$ such that is $i$ bears X for $\omega$ in $(\plan_1, s_1)$.
\end{definition}
We will now show the logical implications between our different forms of responsibility. The horizontal arrows indicate that in any joint plan $\plan$ where $i$ is attributed some form of responsibility, $i$ can anticipate that form of responsibility in the individual plan $\plan^{\{i\}}$. 
\begin{theorem}\label{thm:figure}
The implications shown in figure \ref{figure3} are correct.
\end{theorem}
\begin{figure}[h!]
\begin{center}
\includegraphics[scale = 0.35]{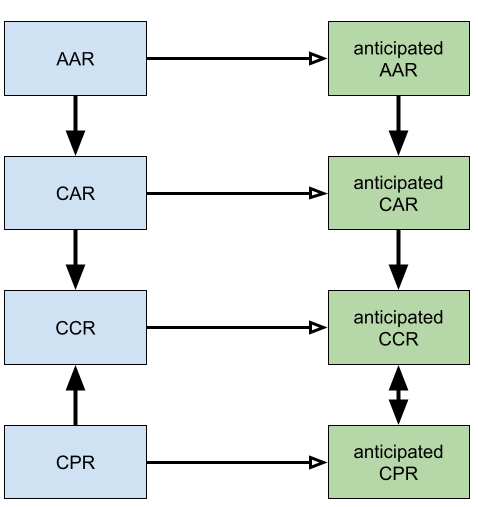}
\caption{A visual representation of the implications between our different forms of responsibility.}\label{figure3}
\end{center}
\end{figure}

Instead of giving a singular modular definition for anticipated responsibility, we can instead give seperate definitions for each notion. For example, consider the following equivalent definition for anticipated Agentive Active Responsibility:

\begin{definition}
Let $\pldomain = (\effprecond, s_0, (\epequiv_i)_{i \in \agentset})$ be a PPD, $i \in \agentset$ an agent, and $\plan$ an individual plan. Let $\omega \in \langlogic_{\ltllogic}$. Then, we say that $i$ anticipates $AAR$ for $\omega$ in $(\plan,\pldomain)$ if for all $s_1 \in \epequiv_i$ and all joint plans $\plan_1$ compatible with $\plan$, $\history^{\plan_1,s_1,\gamma}\models \omega$ and there is some $s_2 \in \epequiv_i$ and some joint plan $\plan_2$ such that $\history^{\plan_2,s_2,\gamma}\models \neg \omega$.
\end{definition}

However, we prefer our modular definition as it emphasises that we have a single notion of anticipated responsibility.

\subsection{Responsibility Anticipation in Plan Selection}
As previously stated, our hypothesis is that anticipating responsibility can help agents to coordinate towards a common goal, even without communication. Given some goal or value $\phi$, agents should avoid active responsibility for $\neg \phi$. This means performing a plan that does not anticipates AAR for $\neg \phi$. Furthermore, we prove that there is always a plan that does not anticipate AAR for $\neg \phi$. This means that artificial agents can be formally verified to never be potentially actively responsible for the violation of some value. This could be a useful step in creating provably safe autonomous planning agents.

\begin{theorem}
Let $\pldomain = (\effprecond, s_0, (\epequiv_i)_{i\in \agentset})$ be a PPD, $i \in \agentset$, and $\omega$ an $\ltllogic$-formula. Then there exists some individual plan $\plan$ for $i$ such that $i$ does not anticipate AAR for $\omega$ in $\plan$.
\end{theorem}

\begin{proof}
(sketch) Either there is some compatible plan where $\omega$ does not occur (meaning $i$ does not anticipate AAR) or $\omega$ occurs in every outcome of every plan, so $i$ is not responsible. 
\end{proof}

Given some value or goal $\phi$, we want agents to avoid responsibility for $\neg \phi$, but also to seek responsibility for $\phi$ (preferably agentice active responsibility, as this guarantees the occurrence of $\phi$). However, we can show that anticipating agentive active responsibility for $\phi$ is effectively equivalent to not anticipating causal passive responsibility for $\neg \phi$ (the dual notion of anticipating CPR). 

\begin{theorem}
Let $\pldomain = (\effprecond, s_0, (\epequiv_i)_{i\in \agentset})$ be a PPD, $i \in \agentset$, and $\omega$ an $\ltllogic$-formula. If there is some plan $\plan$ for $i$ such that $i$ anticipates AAR for $\omega$ in $\plan$, then for any plan $\plan'$ for $i$, $i$ does not anticipate CPR for $\neg \omega$ in $\plan'$ if and only if $i$ anticiptes AAR for $\omega$ in $\plan'$.
\end{theorem}

\begin{proof}
(sketch) Given a joint plan $\plan$ in a planning domain $\pldomain$, $i$ is ``powerless'' with respect to $\omega$ if no alternative plan for $i$ changes the truth value of $\omega$ in $\history^{\plan,s_0,\gamma}$. If $\omega$ occurs in all plans where $i$ is powerless then for all plans $\plan'$ for $i$, $i$ does not anticipate CPR for $\neg \omega$ in $\plan'$ if and only if $i$ anticiptes AAR for $\omega$ in $\plan'$. Otherwise, there is no plan $\plan'$ where $i$ anticiptes AAR for $\omega$ in $\plan'$.
\end{proof}

By ``effectively equivalent'' we mean that if there exists some plan $\plan$ for $i$ that anticipates AAR for $\phi$, then the plans that anticipate AAR for $\phi$ are exactly the plans that do not anticipate CPR for $\neg \phi$. However, the notions are not logically equivalent because it is possible that there are some plans for $i$ that do not anticipate CPR for $\neg \phi$ while there are none that anticipate AAR for $\phi$.

This also suggests that anticipated CPR is the the most important notion of anticipated responsibility, as it is either equivalent or effectively equivalent to every other notion of anticipated responsibility. Finally, we can show that avoiding CPR for $\neg \phi$ is a potentially powerful method for allowing a group of agents to coordinate on a certain goal, even if those agents cannot communicate.

\begin{theorem}
Let $\pldomain = (\effprecond, s_0, (\epequiv_i)_{i\in \agentset})$ be a PPD and $\omega$ an $\ltllogic$-formula. Let $\plan$ be a joint plan such that for every agent $i \in \agentset$, $i$ does not anticipate CPR for $\neg \omega$ in $\plan^{\{i\}}$. Then either $\history' \models \neg \omega$ for every history compatible with $\pldomain$, or $\history^{\plan,s_0,\effprecond} \models \omega$.
\end{theorem}

\begin{proof}
(sketch) Suppose for contradiction that $\omega$ occurs in some plan $\plan'$ and does not occur in $\plan$. Then by theorem \ref{thm:completeness} there is some agent $i$ who bears CCR for $\neg \omega$ in $\plan$. Then by theorem \ref{thm:figure} it must be the case that $i$ anticipates CPR for $\neg \omega$ in $\plan^{\{i\}}$, which is a contradiction.
\end{proof}
This shows that even when agents with a shared goal cannot communicate and when no agent can individually guarantee the success of the goal, the application of anticipated responsibility can allow the agents to successfully coordinate their actions and achieve the goal.

\section{Computing and Implementing Responsibility}\label{sec:theorems}

\subsection{PDDL Implementation}

As previously mentioned, our model is designed to be practically useful in real-world planning problems. Therefore we outline how our model can be implemented in the multi-agent extension of PDDL 3.1 proposed by Kovacs \cite{kovacs2012}.

PDDL solvers take two inputs: a domain and a problem. The domain gives the object types, actions and predicates, whereas the problem gives the objects, initial state and goal. Below is some simplified PDDL code for a multi-agent planning domain involving a number of immobile agents and some tables that is inspired by the example of Kovacs \cite{kovacs2012}. The agents can lift tables that they are next to, or do nothing ($\skipact$). Our example involves two tables (table1 and table2) and two agents (A1 and A2).

\begin{lstlisting}[caption=Example PDDL Domain]
(define (domain responsibility-attribution)

(:requirements :equality :negative preconditions :typing :multi-agent)
(:types agent table)
(:predicates (lifted ?o - object) (at ?a - agent ?o object))
(:action lift :agent ?a - agent :parameters (?o - object)
:precondition (and(not(lifted ?o))
                  (at ?a ?o))
:effect (lifted ?o))

(:action skip :agent ?a - agent :parameters ()
:precondition ()
:effect ()))
\end{lstlisting}

Consider the history where each agent starts next to a separate table, A1 performs the action $\skipact$ and A2 performs $\actlift$. The following code illustrates how we can use PDDL to check if A1 bears CPR for $\omega = \neg \eventually \henceforth($lifted table1 $\land$ lifted table2$)$.

Running the first problem checks if $\omega$ actually occurs, the second problem fixes the actions of all agents besides A1 and checks if A1 could have acted differently and avoided $\omega$. If a plan is found, then A1 bears CPR for $\omega$ (we present just the goal as the rest is the same as the first problem).

\begin{lstlisting}[caption= Checking that the outcome occurs.]
(define (problem causal-passive-responsibility-1)
(:domain responsibility-attribution)
(:objects A1 A2 - agent table1 table2 - table)
(:init (at A1 table1) (at A2 table2))
(:goal (and (lifted table1) (lifted table2)
            (do(A1 skip 1))(do(A2 lift 1)))))
\end{lstlisting}

\begin{lstlisting}[caption= Checking for Causal Passive Responsibility]
(:goal (and (lifted table1) (lifted table2) (do(A2 lift 1)))))
\end{lstlisting}

To describe the plans of agents in PDDL goals we use $\does{i}{a,t}$ which is true whenever agent $i$ does action $a$ at time $t$.\footnote{For simplicity, we do not define $\does{i}{a,t}$ in the code as its definition is quite complex and uninteresting.}

In terms of the outcomes that we can attribute or anticipate responsibility for, PDDL 3 supports any boolean combination of predicates as goals, and also features temporal operators that function as state constraints for writing $\ltllogic$ outcomes \cite{Gerevini05}. However, since PDDL does not support nesting of temporal operators, we do not have the full expressiveness of $\ltllogic$. That said, the expressiveness of PDDL should be sufficient for the vast majority of outcomes that one realistically might be interested in.

The following problems demonstrate how to check CAR for A1 and $\omega$. Firstly, we have to check if $\omega$ is inevitable by attempting to find a joint plan that acheives $\neg \omega$. Then we have to check if the actions of A1 are sufficient to guarantee $\omega$.

\begin{lstlisting}[caption= Checking if $\omega$ is inevitable.]
(:goal (and (lifted table1) (lifted table2))))
\end{lstlisting}

\begin{lstlisting}[caption= CAR Attribution part 2]
(:goal (and (lifted table1) (lifted table2) (do(a skip 1)))))
\end{lstlisting}

For checking AAR we first have to follow the procedure for checking CAR, but then we also have to check that the actions of A1 are sufficient to guarantee $\omega$ in every epistemically possible world for A1. In this example we will suppose that $(\epequiv_{A1}) = \{\{$at A1 table1, at A2 table2$\},\{$at A1 table1, at A2 table1$\}\}$ modelling that A1 does not know where A2 is.

\begin{lstlisting}[caption= AAR Attribution]
(define (problem causal-active-responsibility-2)
(:domain responsibility-attribution)
(:objects a b - agent table1 table2 - table)
(:init (at a table1) (at b table1))
(:goal (and (lifted table1) (lifted table2) (dotime(a skip 1)))))
\end{lstlisting}

For anticipating CAR or AAR the process is much the same as attribution, since attribution depends only on the actions of A1, meaning the actions of all other agents do not need to be defined. We simply have to repeat the procedure for CAR or AAR attribution once for each epistemically possible start state. If A1 bears CAR/AAR in any start state, then they anticipate CAR/AAR. The process for anticipating CPR is more complex. This is because we need to find start state and a plan for all agents besides A1 such that the intented plan for A1 leads to $\omega$ but there exists some other plan for A1 that leads to $\neg \omega$. This can be solved in a single planning problem (at least, one problem per possible start state) by creating a duplicate copy of each object, allowing us to effectively run two copies of the planning domain in parrallel, with the goal enforcing that the actions of all agents besides A1 must be the same in both copies. If a plan is found for any possible start state, then A1 anticipates CPR for $\omega$.

\begin{lstlisting}[caption= CPR Anticipation]
(define (problem causal-active-responsibility-2)
(:domain responsibility-attribution)
(:objects a b a-1 b-1 - agent table1 table2 table1-1 table2-1- table)
(:init (at a table1) (at b table2) (at a-1 table1-1) (at b-1 table2-1))
(:goal (and (lifted table1) (lifted table2) 
            (do(a skip 1))
            (not (and (lifted table1-1) 
                      (lifted table2-1)))
            (henceforth (and (do(b skip) ->  
                                do(b-1 skip))
                             (do(b lift)->
                                do(b-1 lift)))))))
\end{lstlisting}

The procedure for attributing CCR is more complex, as which agents actions we have to fix varies depending on which coalition we are testing, and there are exponentially many coalitions to check. Fortunately, since anticipated CCR is equivalent to anticipated CPR, the procedure for checking that is relatively straightforward.

\subsection{Complexity Results}

In this section we will demonstrate the computational complexity of determining various kinds of blameworthiness. Full proofs of our results can be found in the supplementary material. We define X-ATTRIBUTION as the problem of determining if $i$ bears X$\in \{$CAR,CPR,CCR,AAR$\}$ for $\omega$ in $\plan$ and X-ANTICIPATION as the problem of determining if $i$ \textit{anticipates} X for $\omega$ in $\plan$.

\begin{theorem}\label{thm:attributioncomplexity}
CAR-ATTRIBUTION is a member of P$^{\mathit{NP}[2]}$, CPR-ATTRIBUTION is NP-Complete, CCR-ATTRIBUTION is a member of $\Sigma_2^{P}$ and AAR-ATTRIBUTION is a member of $\Delta_2^{P}$. 
\end{theorem}

\begin{theorem}\label{thm:anticipationcomplexity}
CAR-ANTICIPATION is a member of $\Delta^2_p$, CPR-ANTICIPATION is NP-Complete, CCR-ANTICIPATION is NP-Complete, and AAR-ANTICIPATION is a member of $\Delta_2^{P}$,.
\end{theorem}

These results are only intended to give an introduction to the complexity analysis of this setting. One class of problems that deserve further study is the task of identifying if a plan exists that does/does not anticipate responsibility for some outcome $\omega$ (decision problem) and finding such a plan if one exists (search problem). The problem of identifying if a CPR-anticipating plan exists should be NP-complete given theorem \ref{thm:anticipationcomplexity}, as NP allows us to simply guess a plan, and then check for anticipated responsibility. This puts us in line with the computational complexity of single-agent planning with propositional goals, which is also NP-complete \cite{Turner02}.
\section{Conclusions and Future Work} \label{sec:futurework}

In this paper we have presented our model for responsibility attribution and anticipation in a multi-agent planning setting with partial information regarding the initial state. We have presented both causal and agentive versions of active, passive and contributive responsibility. We have demonstrated how our notions of anticipated responsibility could be useful for plan selection in a multi-agent setting, and have given a complexity analysis of our model. Finally, we have outlined a PDDL implementation of our model.

For future work, a full PDDL implementation would be allow us to test how useful our concepts of responsibility are when applied to real-world planning problems. Furthermore, we could expand our notions of responsibility to handle additional factors. The most obvious extensions would be For example, we could include beliefs about the likely actions of other agents in line with Lorini \cite{LoriniLM14}, or could consider intentions, probabilities and/or degress of responsibility in line with Halpern and Kleiman-Weiner \cite{HalpernK18}. Finally, since agents may have multiple goals or values that they may be held responsible for satisfying or violating, it would be useful to extend our model to allow plan comparison based on anticipated responsibility for multiple different outcomes.


\bibliography{ecai}
\newpage
\appendix
\section{Supplementary Material}

This section contains various proofs and definitions from the paper.

\subsection{Introduction}

No additional material.

\subsection{Related Work}

No additional material.

\subsection{Model}\label{model}

\begin{definition}[Action-compatible histories]
Let
$\effprecond=(\effprecondplus,\effprecondminus)$
be an action theory and let $\history= (\historyst, \historyact)$ be a $k$-history.
We say $\history$
is compatible with $\effprecond$
if the following condition holds
for every $t \in \{0,\ldots,k-1\}$:
\begin{align*}
\historyst(t+1) =& \Big( \historyst (t)  \setminus \big\{p  \in \propset \suchthat \big(\exists i\in \agentset, \exists a \in \actset \text{ such that }\\
&\historyact(i,t)=a
\text{ and } \history, t \models \effprecondminus(i,a,p) \big)  \text{ and }\\ &\big(\forall j\in \agentset, \forall b \in \actset \text{ if } \historyact(j,t)=b \text{ then }\\
& \history, t \models \neg\effprecondplus(j,b,p)\big) \big \}\Big)\\
&\cup \big\{p  \in \propset \suchthat \big(\exists i\in \agentset, \exists a \in \actset 
\text{ such that }\\
&\history, t \models \effprecondplus( i,a,p ) \big) \text{ and } \big(\forall j \in \agentset, \forall b \in \actset \text{ if }\\
&\historyact(j,t)=b \text{ then } \history, t \models \neg\effprecondminus(j,b,p)\big) \big\}.
\end{align*}
\end{definition}

In words, a
history 
$\history$
is a $\effprecond$-compatible history for action theory $\effprecond=(\effprecondplus,\effprecondminus)$ if each state respects the actions performed in the previous state. Propositions become false if the negative effect precondition (for that proposition) of an executed action holds, while the positive effect preconditions of all executed actions do not hold.
Similarly, a proposition becomes true if the positive effect precondition of an executed action holds, while the negative effect preconditions
of all executed actions do not hold. 
In case of conflicts between actions, we use an inertial principle: if one or more actions attempt to enforce different truth values for $p$, then the truth value of $p$ does not change.

\begin{proposition}
Given a state transition function $\tau$, there exists an action theory $\effprecond$ that is equivalent to $\tau$ and at worst polynomially larger.
\end{proposition}

\begin{proof}
\begin{proof}
Let $\tau$ be a state transition function. For any state $s \in \stateset$ let $\phi(s)$ be a conjunction of literals that describes exactly the state $s$. For any joint action $\overline{a} \in \actset^\agentset$ let $\psi(\overline{a})$ be a conjunction of $\does{i}{a}$ that exactly describes $\overline{a}$. For every proposition $p \in \propset$ let $T(p)$ be the set of $(s,\overline{a}) \in \stateset \times \actset^\agentset$ such that $p$ is true in $\tau(s,\overline{a})$. Similarly let $F(p)$ be the set of $(s,\overline{a}) \in \stateset \times \actset^\agentset$ such that $p$ is false in $\tau(s,\overline{a})$ Define $\effprecond$ as follows:

\begin{align*}
    &\effprecondplus(i,a,p) = \bigvee_{(s,\overline{a}) \in T(p)} \phi(s) \land \psi(\overline{a})\\
    &\effprecondminus(i,a,p) = \bigvee_{(s,\overline{a}) \in F(p)} \phi(s) \land \psi(\overline{a})
\end{align*}

To see that $\effprecond$ is equivalent to $\tau$, consider an arbitrary state $s$ and joint action $\overline{a}$. Let $s_1 = \tau(s,\overline{a})$ let $s_2$ be the state following from $s$ and $\overline{a}$ according to $\effprecond$. Suppose $p \in s_1$ for some $p \in \propset$. Therefore $(s,\overline{a}) \in T(p)$. Therefore, by our definition of $\effprecond$, for every $i \in \agentset$ and $a \in \actset$, $s, \overline{a} \models \effprecondplus(i,a,p)$. Furthermore, for every $i \in \agentset$ and $a \in \actset$, $s, \overline{a} \not \models \effprecondminus(i,a,p)$ since the only possible disjunct in $\effprecondminus(i,a,p)$ that can be true is $\phi(s) \land \psi(a)$ which is not in $\effprecondminus(i,a,p)$ since $p$ is not false in $s_1$. Therefore $p \in s_2$. By a similar argument we can see that if $p \notin s_1$ then $p \notin s_2$.

We will now show that $\effprecond$ is at most polynomially larger than $\tau$. Each ``entry'' in $\tau$ contains one joint action and two states, and therefore is roughly of size $(2 \times |\propset|)+|\agentset|$. Since there is an entry in $\tau$ for every possible joint action/state pair, there must be $|\actset|^{|\agentset|} \times 2^{|\propset|}$ entries. Therefore the total size is $((2 \times |\propset|)+|\agentset|) \times |\actset|^{|\agentset|} \times 2^{|\propset|}$. Which is exponential in the size of $\propset$ and $\agentset$.

In the version of $\gamma$ that we have defined, the size of each entry is at most the size of all possible joint action/state pairs. Which is of size $|\actset|^{|\agentset|} \times 2^{|\propset|}$. The number of entries is two (for $\effprecondplus$ and $\effprecondminus$) times the number of action/agent/proposition triples, meaning $2 \times |\actset| \times |\agentset| \times |\propset|$. This gives us a total size of $2 \times |\actset| \times |\agentset| \times |\propset| \times |\actset|^{|\agentset|} \times 2^{|\propset|}$. Cancelling out shared terms gives us a relative size of $(2 \times |\propset|)+|\agentset|)$ for $\tau$ and $2 \times |\actset| \times |\agentset| \times |\propset|$ for $\effprecond$.

However, the size of $\tau$ is always exponential in the size of $\propset$ and $\agentset$, since the number of entries in $\tau$ are fixed. On the other hand, entries for $\effprecond$ can in some cases be as small as constant size (as we can define a class of $\effprecond$ where every $\effprecondplusminus(i,a,p) \in \{\top, \bot\}$). This means $\effprecond$ can be as small as $2 \times |\actset| \times |\agentset| \times |\propset|$.
\end{proof}
\end{proof}

\begin{proposition}
There exists some state transition function $\tau_1$ such that any action theory $\effprecond$ that is equivalent to $\tau_1$ must contain $\does{i}{a}$.
\end{proposition}

\begin{proof}
Consider the following state transition function $\tau_1$ for $\actset = \{a_1,a_2\}$, $\agentset = \{A_1,A_2\}$ and $\propset = \{p\}$.

\begin{center}
\begin{tabular}{c|c|c|c}
    $A_1$ & $A_2$ & $s_0$ & $s_1$ \\
    \hline
    $a_1$ & $a_1$ & any & $p$\\
    $a_1$ & $a_2$ & any & $\neg p$\\
    $a_2$ & $a_1$ & any & $\neg p$\\
    $a_2$ & $a_2$ & any & $p$\\
\end{tabular}
\end{center}

Suppose for contradiction that $\effprecond$ is an action theory that is equivalent to $\tau_1$ and does not contain $\does{i}{a}$. Therefore since $\propset$ is a singleton set, we can suppose without loss of generality that $\effprecondplusminus(i,a,p) \in \{\top,\bot,p,\neg p\}$ for all $i \in \agentset$ and $a \in \actset$.

Firstly, we will show that $\effprecondplusminus(i,a,p)$ must always be in $\{\bot,p,\neg p\}$. If $\effprecondplus(i,a,p) = \top$ then agent $i$ performing action $a$ in state $p$ will always lead to state $p$ by the inertial principle, which contradicts $\tau_1$. Similarly, if $\effprecondminus(i,a,p) = \top$ then agent $i$ performing action $a$ in state $\neg p$ will always lead to state $\neg p$, which also contradicts $\tau_1$.

Since $\tau_1(\neg p, [A_1 \mapsto a_1, A_2 \mapsto a_1]) = p$, it must be the case that for some $i \in \agentset$, $\neg p \models \effprecondplus(i,a_1,p)$, meaning that $\effprecondplus(i,a_1,p) = \neg p$. Without loss of generality, we can suppose that $\effprecondplus(A_1,a_1,p) = \neg p$. We also know that $\effprecondminus(A_1,a_1,p), \effprecondminus(A_2,a_1,p) \in \{\bot, p\}$.

Therefore, since $\tau_1(\neg p, [A_1 \mapsto a_1, A_2 \mapsto a_2]) = \neg p$, it must be the case that either $\neg p \models \effprecondminus(A_1,a_1,p)$ or $\neg p \models \effprecondminus(A_2,a_2,p)$, but we already know that $\effprecondminus(A_1,a_1,p) \in \{\bot, p\}$, so it must be that $\effprecondminus(A_2,a_2,p) = \neg p$. However, this means that the joint action $[A_1 \mapsto a_2, A_2 \mapsto a_2]$ in the state $\neg p$ must result in $\neg p$, which contradicts $\tau_1$.
\end{proof}

\subsection{Formalising Responsibility}\label{sec:anticipating}

\begin{theorem}\label{thm:completeness}
Let $\pldomain=(\effprecond, s_0,(\epequiv_i)_{i\in \agentset})$ be a PPD, let $\plan$ be a joint plan and let $\history = \history^{\plan,s_0,\effprecond}$. Let $\omega \in \langlogic_{\ltllogic}$ such that $\history \models \omega$. Then either $\history' \models \omega$ for every history compatible with $\pldomain$, or there exists some $i \in \agentset$ such that $i$ bears CCR for $\omega$ in $\plan$.
\end{theorem}

\begin{proof}
Let $\pldomain=(\effprecond, s_0,(\epequiv_i)_{i\in \agentset})$ be a PPD, let $\plan$ be a joint plan and let $\history = \history^{\plan,s_0,\effprecond}$. Let $\omega \in \langlogic_{\ltllogic}$ such that $\history \models \omega$. Suppose that there exists some history $\history'$ compatible with $\pldomain$ such that $\history' \models \neg\omega$. We will now show that there exists some $i \in \agentset$ such that $i$ bears CCR for $\omega$ in $\plan$.

Let $\agentset = \{i_1,\ldots,i_n\}$. Fixing the actions of the coalition $\agentset$ exactly determines the history, so the actions of $\agentset$ are sufficient to guarantee $\omega$. However, the actions of the coalition $\emptyset$ are not sufficient to guarantee $\omega$ given that $\history' \models \neg\omega$. Now consider the sequence of coalitions $\agentset, \agentset \setminus \{i_1\},\ldots,\{i_n\},\emptyset$. Then there must exist some pair of coalitions $J, J \setminus \{i_m\}$ in the sequence such that the actions of $J$ guarantee $\omega$ but the actions of $J \setminus \{i_m\}$ do not. Therefore $i_m$ bears CCR for $\omega$ in $\plan$.
\end{proof}

\begin{theorem}\label{thm:figure}
The implications shown in figure \ref{figure3} are correct.
\end{theorem}

\begin{figure}[h!]
\begin{center}
\includegraphics[scale = 0.35]{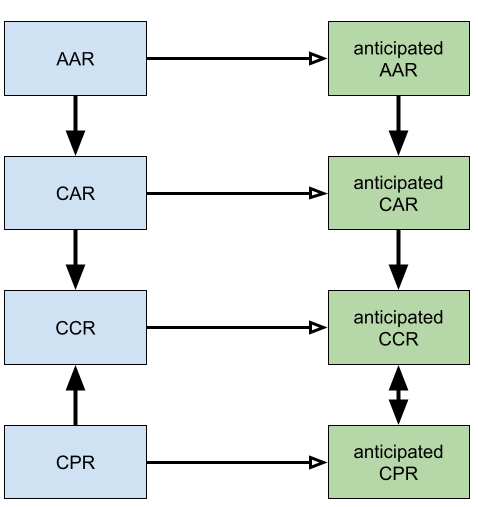}
\caption{A visual representation of the implications between our different forms of responsibility.}\label{figure3}
\end{center}
\end{figure}

\begin{proof}
Firstly, all horizontal arrows are trivially true from the definition of anticipated responsibility.

\paragraph{(AAR $\Rightarrow$ CAR)} True from definitions.

\paragraph{(CAR $\Rightarrow$ CCR)} Suppose $i$ beasrs CAR for $\omega$ in $\plan$. Then the coalition $J = \{i\}$ meets the requirements for CCR since $J \setminus \{i\} = \emptyset$ and since all possible plans are compatible with $\plan^{\{i\}}$ and CAR ensures that there is some plan where $\omega$ does not occur.

\paragraph{(CPR $\Rightarrow$ CCR)} Suppose $i$ bears CAR for $\omega$ in $\plan$. Then the coalition $J = \agentset$ meets the requirements for CCR.

\paragraph{}Note that for responsibility notions X and Y, if X $\Rightarrow$ Y then it must also be the case that anticipated X $\Rightarrow$ anticipated Y since any outcome where $i$ anticipates $X$ must also be one where they anticipate $Y$. Therefore the only remaining proof is:

\paragraph{(anticipated CCR $\Rightarrow$ anticipated CPR)} Suppose $i$ anticipates CCR for $\omega$ in $\plan$. Then there is some joint plan $\plan_1$ compatible with $\plan$ and some coalition $J$ including $i$ such that the actions of $J$  in $\plan_1$ are sufficient to guarantee $\omega$ but the actions of $J \setminus \{i\}$ are not. Therefore there exists some plan $\plan_2$ compatible with $\plan_1^{J \setminus \{i\}}$ where $\omega$ does not occur. Let $\plan_3 = \plan_2^{\agentset \setminus \{i\}} \cup \plan$. Then $i$ bears CPR for $\omega$ in $\plan_3$ and therefore $i$ anticipates CPR in $\plan$.

\end{proof}

\begin{theorem}
Let $\pldomain = (\effprecond, s_0, (\epequiv_i)_{i\in \agentset})$ be a PPD, $i \in \agentset$, and $\omega$ an $\ltllogic$-formula. Then there exists some individual plan $\plan$ for $i$ such that $i$ does not bear AIAR for $\omega$ in $\plan$.
\end{theorem}

\begin{proof}
Suppose that there exists some joint plan $\plan$ and some initial state $s_1 \in (\epequiv_i)$ such that $\history^{\plan,s_0,\effprecond} \models \neg \omega$. Then $i$ does not bear AIAR for $\omega$ in $\plan^{\{i\}}$.

Suppose instead that in all joint plans $\plan$ and all initial states $s_1 \in (\epequiv_i)$, $\history^{\plan,s_0,\effprecond} \models \omega$. Then $i$ does not bear IAR for $\omega$ in any joint plan $\plan$, and therefore does not bear AIAR for $\omega$ in any individual plan $\plan^{\{i\}}$.
\end{proof}

\begin{definition}[Powerlessness]
Let $\pldomain = (\effprecond, s_0, (\epequiv_i)_{i\in \agentset})$ be a PPD, $\plan$ a joint plan, $i \in \agentset$, and $\omega$ an $\ltllogic$-formula. Then we say that $i$ is powerless with respect to $\omega$ in $\phi$ if changing the actions of $i$ in $\plan$ does not affect the value of $\omega$.
\end{definition}

\begin{theorem}
Let $\pldomain = (\effprecond, s_0, (\epequiv_i)_{i\in \agentset})$ be a PPD, $i \in \agentset$, and $\omega$ an $\ltllogic$-formula. Then if there exists some individual plan $\plan$ for $i$ such that $i$ bears AIAR for $\neg \omega$ in $\plan$, for any plan $\plan'$ for $i$, $i$ does not bear AIPR for $\omega$ in $\plan'$ if and only if $i$ bears AIAR for $\neg \omega$ in $\plan'$.
\end{theorem}

\begin{proof}
Let $\pldomain = (\effprecond, s_0, (\epequiv_i)_{i\in \agentset})$ be a PPD, $i \in \agentset$, and $\omega$ an $\ltllogic$-formula. Suppose there exists some individual plan $\plan$ for $i$ such that $i$ bears AIAR for $\neg \omega$ in $\plan$. Therefore there must be some joint plan $\plan \cup \plan''$ where $i$ is not powerless with respect to $\omega$. Furthermore, in every joint plan $\plan \cup \plan''$ where $i$ is powerless with respect to $\omega$, $\neg \omega$ occurs. 

\paragraph{$\Rightarrow$} Suppose that for some plan $\plan'$, $i$ does not bear AIPR for $\omega$ in $\plan'$. Then in every joint plan $\plan' \cup \plan''$ where $i$ is not powerless, $\neg \omega$ occurs. Since we know that $\neg \omega$ also occurs whenever $i$ \textit{is} powerless, we can conclude that $i$ bears AIAR for $\neg \omega$.

\paragraph{$\Leftarrow$} Suppose that for some plan $\plan'$, $i$ bears AIAR for $\neg \omega$ in $\plan'$. Then in every joint plan $\plan' \cup \plan''$ where $i$ is not powerless, $\neg \omega$ occurs. Therefore we can conclude that $i$ does not bear AIPR for $\omega$.
\end{proof}

\begin{theorem}
Let $\pldomain = (\effprecond, s_0, (\epequiv_i)_{i\in \agentset})$ be a PPD and $\omega$ an $\ltllogic$-formula. Let $\plan$ be a joint plan such that for every agent $i \in \agentset$, $i$ does not bear AIPR for $\omega$ in $\plan^{\{i\}}$. Then either $\history' \models \omega$ for every history compatible with $\pldomain$, or $\history^{\plan,s_0,\effprecond} \models \neg \omega$.
\end{theorem}

\begin{proof}
Let $\pldomain = (\effprecond, s_0, (\epequiv_i)_{i\in \agentset})$ be a PPD and $\omega$ an $\ltllogic$-formula. Let $\plan$ be a joint plan such that for every agent $i \in \agentset$, $i$ does not bear AIPR for $\omega$ in $\plan^{\{i\}}$. Suppose that $\history' \models \omega$ for some history compatible with $\pldomain$.

Suppose for contradiction that $\history^{\plan,s_0,\effprecond} \models \omega$. Then by theorem \ref{thm:completeness}, it must be the case that some agent $i$ bears CCR for $\omega$ in $\plan$. By theorem \ref{thm:figure} this means that $i$ bears AIPR for $\omega$ in $\plan^{\{i\}}$, which is a contradiction.

\end{proof}

\subsection{Complexity Results and Algorithms}\label{sec:theorems}

\subsubsection{Preliminary Results}




\begin{cproblem}{LTLf-SATSET}
INPUT: A PD $\nabla = (\effprecond, s_0)$ for a set of agents $\agentset$, a joint plan $\plan$, $\phi\in \langlogic_{\ltllogic+}$.\\
QUESTION: Is $\phi$ satisfied by the execution of $\plan$ in $\nabla$?
\end{cproblem}

\begin{theorem*}
LTLf-SATSET is in P
\end{theorem*}

\begin{proof}
To show that LTLf-SATSET is in P, we present a basic algorithm that uses polynomial time.

We begin by showing that generating the unique history associated to a joint plan can be done in polynomial time. Then we show that evaluating an $\ltllogic+$ formula over this history can be done in polynomial time.

The $\historyact$ associated to $\plan$ is effectively just $\plan$, so this can be generated in polynomial time. We then set $\historyst(0) = s_0$, then for each $\historyst(i)$ we use $\gamma$ to generate $\historyst(i+1)$ in polynomial time by model checking all formulas in $\gamma$. We need to repeat this process a number of times linear with the size of the input (specifically, the length of $\plan$).

Let us now show that $\ltllogic+$ formulas can be checked in polynomial time on $\history$. Let $n$ be the length of $\phi$. We proceed by strong induction on $n$. 
Let $\history_j$ be the history such that $\historyactj(i) = \historyact(i+j)$ and $\historystj(i) = \historyst(i+j)$.

\textbf{Base case.} Suppose n = 1, then either $\phi = p$ for some $p \in \propset$ or $\phi = do(i,a)$ for some agent $i \in \agentset$ and action $a \in \actset$. We can determine if $p \in \historyst(0)$ in polynomial time, and we can determine if $\historyact(i,0) = a$ in polynomial time.

\textbf{Inductive step.} Suppose $n > 1$ and that the claim holds for all $m < n$. Then we have several options for $\phi$.

\begin{enumerate}
    \item $\phi = \neg \psi$. Then by inductive hypothesis we can determine in polynomial time if $H \models \psi$ and thus if $H \models \phi$.
    \item $\phi = \psi \land \chi$. Then we can determine in time $P$ if $\history \models \psi$ and $\history \models \chi$.
    \item $\phi = \nexttime \psi$. Then we can determine (in polynomial time) if $\history_1 \models \psi$.
    \item $\phi = \until{\psi}{\chi}$. Then for $0 < j < k$ we can determine if $\history_0, \history_1, ... , \history_{j-1} \models \psi$ and $\history_j \models \chi$ in polynomial time. Therefore this whole process can be done in polynomial time.
\end{enumerate}
\end{proof}

\begin{cproblem}{LTLf-MA-PLANMIN-POLY}
INPUT:  An integer $k$ whose value is polynomially large relative to the input size, a set of agents $\agentset$, a PD $\nabla = (\effprecond, s_0)$ and a $\langlogic_{\ltllogic+}$ formula $\phi$.\\
QUESTION: Does there exist a plan $\Pi$ of length at most $k$ such that $\history^{\plan,s_0,\effprecond}, 0 \models \phi?$
\end{cproblem}

\begin{theorem*}
LTLf-MA-PLANMIN-POLY is NP-complete.
\end{theorem*}

\begin{proof}
To show that LTLf-MA-PLANMIN-POLY is NP-hard we show a reduction from SAT \cite{Cook71}. Given an instance of SAT $\phi$. We will create $\delta = (k,\agentset,(\effprecond,s_0),\phi')$ such that LFLf-MA-PLANMIN-POLY on $\delta$ is equivalent to SAT on $\phi$.

Let $\agentset = \{A_1\}$. Let $s_0 = \emptyset$ and let $\phi' = \eventually \henceforth \phi$. Let $\actset$ = $\{\mathit{Set }p : p \in \propset\} \cup \{\skipact\}$. We then define $\effprecond$ as follows:
\begin{align*}
    &\effprecondplus(A_1,\mathit{Set }p,p) = \top \text{ for all } p \in \propset\\
    &\effprecondplusminus(A_1,\mathit{Set }p,p') = \bot \text{ unless specified above}
\end{align*}

Let $\propset$ be the set of propositions appearing in $\phi$, set $k = |\propset|$. This ensures that $k$ is polynomially large relative to $\delta$.

To show that LTLf-MA-PLANMIN-POLY is NP-complete we present an NP algorithm. Let $\delta = (k,\agentset,(\effprecond,s_0),\phi)$ be an instance of LTLf-MA-PLANMIN-POLY. Then guess a joint plan $\plan$ of length at most $k$. We know that the size of $\plan$ must be polynomial in the size of the input as the value of $k$ is polynomial in the size of the input.

We know from LTLf-SATSET that we can verify in polynomial time if $\history^{\plan,s_0,\effprecond} \models \phi$. If this is the case, then output $\true$, otherwise output $\false$.
\end{proof}

\subsubsection{Responsibility Attribution}

\begin{cproblem}{X-ATTRIBUTION}
VARIANTS: X $\in \{$CAR, CPR, CCR, AAR$\}$\\
INPUT: A PPD $\pldomain=(\effprecond, s_0,(\epequiv_i)_{i\in \agentset})$ for a set of agents $\agentset$, a joint plan $\plan$, and agent $i$ and a formula $\omega \in \langlogic_{\ltllogic}$.\\
QUESTION: Does $i$ bear X for $\omega$ in $\plan$? 
\end{cproblem}

\begin{theorem}
CPR-ATTRIBUTION is NP-Complete
\end{theorem}

\begin{proof}
To show that CPR-ATTRIBUTION is NP-hard we show a reduction from SAT. Given an instance of SAT $\phi$. We will create $\delta = ((\effprecond,s_0,(\epequiv_i)_{i \in \agentset}),\plan,i,\omega)$ such that CPR-ATTRIBUTION on $\delta'$ is equivalent to SAT on $\phi$.

Let $\propset$ be the set of propositions in $\phi$ plus the new proposition $p_0$. Let $\agentset = \{A_1\}$. Let $\actset$ = $\{\mathit{Set }p : p \in \propset \setminus \{p_0\}\} \cup \{\skipact, \actfail\}$. We then define $\effprecond$ as follows:
\begin{align*}
    &\effprecondplus(A_1,\mathit{Set }p,p) = \top \text{ for all } p \in \propset\\
    &\effprecondplusminus(A_1,\mathit{Set }p,p') = \bot \text{ unless specified above}\\
    &\effprecondplus(A_1, \actfail, p_0) = \top\\
    &\effprecondplusminus(A_1,\actfail,p) = \bot \text{ unless specified above}
\end{align*}

Let $s_0 = \emptyset$. Let $\epequiv_{A_1} = \{s_0\}$. Let $k = |\propset| - 1$. Then the value of $k$ is polynomial in the size of $\phi$. Let $\plan$ be the $k$-plan such that $A_1$ only does $\actfail$. Let $\omega = \phi \land \henceforth \neg p_0$.

To show that CPR-ATTRIBUTION is NP-complete we will present an NP algorithm. Given an instance of CPR-ATTRIBUTION $\delta = ((\effprecond, s_0,(\epequiv_i)_{i\in \agentset}),\plan,i,\omega)$, first check if $\history^{\plan,s_0,\gamma} \models \omega$. This can be done in polynomial time by use of LTLf-SATSET. If the check fails then output $\false$ and we are done.

Otherwise guess an alternative plan $\plan'$ for $i$, of the same length as $\plan$. Let $\plan_1 = \plan' \cup \plan^{-\{i\}}$, then check if $\history^{\plan_1,s_0,\gamma} \models \omega$. If yes then output $\false$, otherwise output $\true$.
\end{proof}

\begin{theorem*}
CAR-ATTRIBUTION is a member of P$^{\mathit{NP}[2]}$
\end{theorem*}

\begin{proof}
To show that CAR-ATTRIBUTION is a member of P$^{\mathit{NP}[2]}$ we present a P$^{\mathit{NP}[2]}$ algorithm.

Given an instance of CAR-ATTRIBUTION $\delta = ((\effprecond, s_0,(\epequiv_i)_{i\in \agentset}),\plan,i,\omega)$, first check (in polynomial time) if $\history^{\plan,s_0,\gamma} \models \omega$. If the check fails then output $\false$ and we are done.

Otherwise, check if $\omega$ occurs in every joint plan. This can be done by the NP Oracle using LTLf-MA-PLANMIN-POLY with $\phi = \neg \omega$ and $k$ equal to the length of $\plan$. Third, check if $\omega$ occurs is every joint plan where the actions of $i$ are fixed. This can again be done by an NP oracle using LTLf-MA-PLANMIN-POLY with $\phi = \neg \omega \land \does{i}{a_1} \land \ldots$ in order to fix the actions of $i$.
\end{proof}

\begin{theorem*}
CCR-ATTRIBUTION is a member of $\Sigma_2^p$.
\end{theorem*}

\begin{proof}
To show that CCR-ATTRIBUTION is a member of $\Sigma_2^p$ we present a $\Sigma_2^p$ algorithm.

Given an instance of CCR-ATTRIBUTION $\delta = ((\effprecond, s_0,(\epequiv_i)_{i\in \agentset}),\plan,i,\omega)$, first check (in polynomial time) if $\history^{\plan,s_0,\gamma} \models \omega$. If the check fails then output $\false$ and we are done.

Otherwise, guess a coalition of agents $J$ that includes $i$. Use the NP oracle with LTLf-MA-PLANMIN-POLY to check if $\omega$ is guaranteed when we fix the actions of $J$. If it is, check if $\omega$ is still guaranteed if we only fix the actions of $J \setminus \{i\}$.
\end{proof}

\begin{theorem*}
AAR-ATTRIBUTION is a member of $\Delta_2^p$.
\end{theorem*}

\begin{proof}
To show that AAR-ATTRIBUTION is a member of $\Delta_2^p$ we present a $\Sigma_2^p$ algorithm.

Given an instance of AAR-ATTRIBUTION $\delta = ((\effprecond, s_0,(\epequiv_i)_{i\in \agentset}),\plan,i,\omega)$, follow the algorithm for CAR-ATTRIBUTION to check if $i$ bears CAR for $\omega$.

If yes, then use a polynomial number of calls to the NP oracle (one for each state in $\epequiv_i$ to check if the actions of $i$ guarantee $\omega$ in every possible start state.
\end{proof}

\subsubsection{Anticipating Responsibility}

\begin{cproblem}{X-ANTICIPATION}
VARIANTS: X $\in \{$CAR, CPR, CCR, AAR$\}$\\
INPUT: A PPD $\pldomain=(\effprecond, s_0,(\epequiv_i)_{i\in \agentset})$ for a set of agents $\agentset$, a plan $\plan$ for agent $i$ and a formula $\omega \in \langlogic_{\ltllogic}$.\\
QUESTION: Does $i$ anticipate X for $\omega$ in $\plan$? 
\end{cproblem}

\begin{theorem}\label{thm:CPRANTICIPATION}
CPR-ANTICIPATION is NP-Complete
\end{theorem}

\begin{proof}

To show that CPR-ANTICIPATION is NP-hard we show a reduction from SAT. We can use exactly the same reduction as for CPR-ATTRIBUTION since anticipation and attribution are equivalent in the single-agent case.

To show that CPR-ANTICIPATION is NP-complete we present an NP algorithm. Given an instance of CPR-ANTICIPATION $\delta = ((\effprecond, s_0,(\epequiv_i)_{i\in \agentset}),\plan,i,\omega)$, first guess a plan $\plan'$ for $\agentset \setminus \{i\}$ and a state $s_1 \in \epequiv_i$. Then perform the algorithm for CPR-ATTRIBUTION on $\plan \cup \plan'$ and $s_1$.
\end{proof}

\begin{theorem*}
CAR-ANTICIPATION is a member of $\Delta_2^p$.
\end{theorem*}

\begin{proof}
To show that CAR-ANTICIPATION is in $\Delta_2^p$ we present a $\Delta_2^p$ algorithm. Given an instance of CAR-ANTICIPATION $\delta = ((\effprecond, s_0,(\epequiv_i)_{i\in \agentset}),\plan,i,\omega)$, for each $s_1 \in \epequiv_i$, we can check if $i$ anticipates CAR in $\plan$ from $s_1$ with two calls to the NP oracle, one to check if $\omega$ is inevitable, and one to check if the actions of $i$ guarantee $\omega$.
\end{proof}

\begin{theorem*}
CCR-ANTICIPATION is NP-complete.
\end{theorem*}

\begin{proof}
By theorem \ref{thm:figure} we know that CCR-ANTICIPATION is equivalent to CPR-ANTICIPATION, which we know to be NP-complete (theorem \ref{thm:CPRANTICIPATION}).
\end{proof}

\begin{theorem*}
AAR-ANTICIPATION is a member of $\Delta_2^p$.
\end{theorem*}

\begin{proof}
To show that AAR-ANTICIPATION is in $\Delta_2^p$ we present a $\Delta_2^p$ algorithm. Given an instance of CAR-ANTICIPATION $\delta = ((\effprecond, s_0,(\epequiv_i)_{i\in \agentset}),\plan,i,\omega)$, for each $s_1 \in \epequiv_i$, we can check if $i$ anticipates CAR in $\plan$ from $s_1$ with polynomially many calls to the NP oracle, one to check if $\omega$ is inevitable from $s_1$, and one to check if the actions of $i$ guarantee $\omega$ in each $s_2 \in \epequiv_i$.
\end{proof}

\end{document}